\documentclass[twoside,twocolumn]{article}

\usepackage{amssymb,graphicx,url,amsmath,subfigure
}

\usepackage{times}
\usepackage{tikz}
\usepackage{soul}
\usepackage{color}
\usepackage{xcolor}
\usepackage{algorithm}
\usepackage{algorithmic}
    \usepackage{enumerate}

\usepackage{amsthm}

\usetikzlibrary{shapes.geometric,positioning}

\newcommand{\Exp}{\mathbb{E}}

\newcommand{\cPA}{{\cal PA}}

\newcommand{\R}{{\mathbb R}}

\newcommand{\cX}{{\cal X}}
\newcommand{\cY}{{\cal Y}}

\newcommand{\bx}{{\bf x}}

\newcommand{\bX}{{\bf X}}

\newcommand{\cN}{{\cal N}}

\newcommand{\bn}{{\bf n}}

\newcommand{\bN}{{\bf N}}

\newtheorem{Theorem}{Theorem}
\newtheorem{Lemma}{Lemma}
\newtheorem{Definition}{Definition}
\newtheorem{Example}{Example}

\newcommand\independent{\protect\mathpalette{\protect\independenT}{\perp}}
\def\independenT#1#2{\mathrel{\rlap{$#1#2$}\mkern2mu{#1#2}}}

\definecolor{MyDarkGreen}{rgb}{0.17,0.46,0.25} 
\definecolor{MyDarkRed}{rgb}{0.88,0.22,0.21} 
\definecolor{MyDarkBlue}{rgb}{0.11,0.11,0.70}

\definecolor{lightgray}{gray}{0.85}
\sethlcolor{Dandelion}

\usetikzlibrary{arrows,shapes,plotmarks,positioning}
\usetikzlibrary{decorations.markings}

\tikzset{>=stealth'} 
\tikzstyle{graphnode} = 
   [circle,draw=black,minimum size=22pt,text centered,text
     width=22pt,inner sep=0pt] 
\tikzstyle{var}   =[graphnode,fill=white]
\tikzstyle{vardashed}   =[graphnode,draw=gray,fill=white]
\tikzstyle{obs}   =[graphnode,fill=black,text=white]
\tikzstyle{obsgrey}   =[graphnode,draw=white,fill=lightgray,text=black]
\tikzstyle{par}    =[graphnode,draw=white,fill=red,text=black] 
 \tikzstyle{crucial} =[graphnode,draw=white,fill=yellow,text=black] 
\tikzstyle{fac}   =[rectangle,draw=black,fill=black!25,minimum size=5pt]
\tikzstyle{facprior} =[rectangle,draw=black,fill=black,text=white,minimum size=5pt]
\tikzstyle{edge}  =[draw=white,double=black,very thick,-]
\tikzstyle{blueedge}  =[draw=white,double=blue,very thick,-]
\tikzstyle{rededge}  =[draw=white,double=red,very thick,-]
\tikzstyle{prior} =[rectangle, draw=black, fill=black, minimum size=
5pt, inner sep=0pt]
\tikzstyle{dirprior} = [circle, draw=black, fill=black, minimum
size=5pt, inner sep=0pt]

\tikzstyle{dot_node}=[draw=black,fill=black,shape=circle]

\date{5 December 2019}

\begin{document}
\frenchspacing

\title{{\bf Causal structure based root cause analysis of outliers }} 

\author{Dominik Janzing, Kailash Budhathoki, Lenon Minorics, and Patrick Bl\"obaum\\
{\small Amazon Research T\"ubingen, Germany }\\
{\small \{janzind, minorics, bloebp\}@amazon.com, kailash.buki@gmail.com} }

\maketitle

\begin{abstract}
We describe a formal approach to identify `root causes' of outliers observed in $n$ variables
$X_1,\dots,X_n$ 
in a scenario where the causal relation between the variables is a known directed acyclic graph (DAG). 
To this end, we first introduce a systematic way to define outlier scores. Further, we introduce the concept of
`conditional outlier score' which measures whether a value of some variable is unexpected {\it given the value of its parents} in the DAG, if one were to assume that the causal structure and the corresponding conditional distributions are also valid for the anomaly. 
Finally, we quantify to what extent the high outlier score of some target variable can be attributed to outliers of its ancestors.
This quantification is defined via {\it Shapley values} from cooperative game theory.  
\end{abstract}

\section{Introduction}
Although asking for `the root cause' of an unexpected event seems to be at the heart of the human way of `understanding' the world, there is no obvious way to make sense of the concept of a `root cause' using current formalizations of causality
like Causal Bayesian Networks (CBNs) \cite{Pearl:00,Spirtes1993}. To elaborate on the kind of causal information provided
by CBNs, we first recall the causal Markov condition. If the variables $X_1,\dots,X_n$ are causally connected by the DAG $G$, then their joint distribution factorizes according to $G$ 
\[
P_{X_1,\dots,X_n} = \prod_{j=1}^n P_{X_j|PA_j},
\]
where each $P_{X_j|PA_j}$ denotes the conditional distribution of each $X_j$, given its parent variables $PA_j$. 
Whenever $P_{X_1,\dots,X_n}$ 
has a density with respect to a product measure,
the joint
density decomposes into \cite{Lauritzen} 
\begin{equation}\label{eq:gml}
p(x_1,\dots,x_n) = \prod_{j=1}^n p(x_j|pa_j).
\end{equation}
Given that the variable $X_j$ attained some `extreme' value $x_j$, there is no straightforward answer to the question `why did this happen?'. After all, $p(x_j|pa_j)$ only describes the conditional probability for this event, given the values $pa_j$ attained by the parents $PA_j$. At the same time, it also describes the {\it interventional} probability of $x_j$, given that $PA_j$ are set to the values $pa_j$.

More detailed causal information is provided if the causal DAG not only comes with a joint distribution $P_{X_1,\dots,X_n}$ 
but also with a 
 {\it functional causal models (FCMs)} \cite{Pearl:00}, also called `structural equation models'. There, each variable is given by a function of
 its parents and an unobserved noise variable
\begin{equation}\label{eq:fcm}
X_j = f_j(PA_j,N_j), 
\end{equation}
where the noise variables $N_1,\dots,N_n$ are jointly statistically independent.   \eqref{eq:fcm} is said to be an FCM {\it for the conditional} $P_{X_j|PA_j}$ 
if $f_j(pa_j,N_j)$ has the distribution $P_{X_j|PA_j=pa_j}$ for $P_{PA_j}$-almost all $pa_j\in PA_j$.
FCMs also answer counterfactual questions referring to what would have happened for {\it that particular observation} $x_j$, given one had intervened on $PA_j$ and set them to $pa_j'$ instead of $pa_j$.  
FCMs are therefore particularly helpful for understanding what happened for that particular event rather than
just providing statistical statements over the entire sample. 

Here we focus on events that particularly require an `explanation'  
because they are
 {\em rare} events.  Such events can be values $x_j$ or combinations $(x_1,\dots,x_n)$ that
 belong to an a priori defined class of events having low probability. Such events are also called `anomalies' or `outliers' or also `extreme events' (for the special case where values are exceptionally large or small). 
Outlier detection is a field of growing 
interest \cite{Chandola2009,Guha2016} particularly in the context of understanding complex systems like global climate \cite{Zscheischler2018}, monitoring of cloud computing networks, health monitoring, and fraud detection, just to name a few. The goal of this paper is not to provide yet another tool for outlier detection. Instead, it discusses an approach to
infer root causes of outlier events given that the outlier detection problem has been solved and that the
causal DAG is available. Here, the causal DAG is assumed not to only describe causal relations in the `normal' regime, but also in the case of the anomalies under consideration. 

The paper is structured as follows. Section~\ref{sec:it} tries to introduce a systematic quantification of outliers and introduces the class of `information theoretic' outlier scores. 
Section~\ref{sec:weak} motivates why this class is particularly helpful for studying causal relations between outliers. 
Section~\ref{sec:co} introduces the concept of {\it conditional} outliers as crucial basis for root cause analysis. Section~\ref{sec:attr} describes how to attribute outliers of
some target variable of interest to each single ancestor node. Section~\ref{sec:exp} describes experiments with real and simulated data.

\section{Information theoretic outlier scores \label{sec:it}}
This section introduces some terminology and tools. 
Although they are not really standard, we do not claim substantial novelty for the major part. 
\begin{Definition}[information theoretic scores]\label{def:score}
Let $X$ be a random variable with values in $\cX$ and distribution $P_X$. An information theoretic (IT) outlier score is a measurable function $S_X: \cX \to \R^+_0$ 
such that 
\begin{eqnarray}
P \{ S_X(X) \geq c\} &\leq &  e^{-c}  \quad \forall c \in \R^+_0  \label{eq:c}\\
P\{S_X(X) \geq S_X(x)\} & = &  e^{-S_X(x)}   \nonumber \\
 \hbox{ for $P_X$-almost all } &&  x \in \cX. \label{eq:cimage}
\end{eqnarray}
$S_X$ is called surjective if the function $S_X:\cX \to \R_0^+$ is surjective and thus equality holds in \eqref{eq:c}.
\end{Definition}
We will often drop random variables as indices in $S_X$ whenever it is clear from the context to which variable we refer to. 

For the following reason surjective scores are particularly convenient
(but they cannot exist for purely discrete probability space): 
\begin{Lemma}[distribution of surjective scores]\label{lem:sur}
For any surjective $S$, the distribution of $S(X)$ on $[0,\infty)$ has the density $p(s)=e^{-s}$. 
\end{Lemma}
\begin{proof}
If $S$ is surjective \eqref{eq:cimage} holds for every $c\in \R_0^+$ and thus the cumulative distribution function reads
$P\{S(X)\leq s\}  = 1- e^{-s}$. Its derivative is $p(s)=e^{-s}$.  
\end{proof}

The fact that the probability decays exponentially with increasing score will be convenient later because it ensures that
scores behaves almost additive for independent events. 
It turns out that IT scores are just generalized log quantiles:  
\begin{Lemma}[function-based representation]\label{lem:funcform}
Every information theoretic outlier score is of the following form.
If $f:\cX \rightarrow \R$ is a measurable function, let $S$ be defined by
\[
S^f (x):= -\log  P\{ f(X) \geq f(x)\}, 
\] 
where $x\in \cX$ denotes an arbitrary value attained by $X$.  
\end{Lemma}
\begin{proof} Set $f:=S$. Using property \eqref{eq:cimage} we have
$P\{S(X)\geq S(x)\} = e^{-S(x)}$ and hence $S(x)= -\log P\{f(X) \geq f(x)\}$.
\end{proof}

\begin{Example}[Rarity]\label{ex:rar}
Let $P_X$ have the density $p_X$. Setting $f(x):= -\log p(x)$ yields the Rarity
\[
S^f(x) = -\log P\{ p(X) \leq p(x)\}.
\]
\end{Example}

\begin{Example}[usual log quantiles]\label{ex:quantile}
For any real-valued $X$ set $f(x):= x$. Then we obtain the right-sided quantile 
\[
S^f(x)= -\log P\{ X\geq x\} . 
\]
For $f(x):=-x$ we obtain the left-sided quantile
\[
S^f(x) = -\log P\{ X\leq x\},
\]
and for $f(x) := |x-\Exp[X]|$ we obtain
\[
S^f(x) = -\log P \left\{ |X - \Exp[X]| \geq |x- \Exp[X]| \right\}.
\]
\end{Example}
The following concept describes a way to generate an outlier score from a set of scores:
\begin{Definition}[scores from convolution]
For any probability distribution on $\cX_1\times \cX_2$ and two functions $f_1:\cX_1\to \R$ and $f_2:\cX_2\to\R$, we call
\begin{align*}
&S^{f_1+f_2} ((x_1,x_2)) :=  \\
&- \log P\{ f_1(X_1) + f_2(X_2) \geq f_1(x_1) + f_2(x_2) \},
\end{align*}
the convolution score of $f_1$ and $f_2$. Similarly, we define the convolution for $n$ components. 

For the special case where $f_1$ and $f_2$ are information theoretic scores, we call the resulting score the {\em convolution of the scores} $S_1$ and $S_2$:
\begin{align*}
&(S_1 * S_2) ((x_1,x_2))  :=\\
&-\log P  \{S_1(X_1) + S_2(X_2) \geq S_1(x_1) + S_2(x_2) \}.
\end{align*}
\end{Definition}

\begin{Example}[Rarity with product densities]
For two random variables $X_1,X_2$ we define  $\bX :=(X_1,X_2)$ and $\bx:=(x_1,x_2)$. Let $X_1$ and $X_2$ be independent and have the joint density
 $p(x_1,x_2)=p_1(x_1)p_2(x_2)$. 
Then the rarity
\[
S(\bx) = -\log P \{ p(\bX) \leq p(\bx) \}, 
\]
is the convolution score of $f_1$ and $f_2$ with $f_j(x_j):=-\log p_j(x_j)$, but not the convolution score of the separate rarities of $X_1$ and $X_2$.
The latter would be given by the function $f':=S_{X_1} + S_{X_2}$, that is the sum of the logarithms of {\em cumulative} tail distributions and not the sum of log {\em densities}.
\end{Example}

\begin{Theorem}[multiple convolution]\label{thm:conv}
Let $P$ be a product distribution on $\cX:=\cX_1\times \cdots \times \cX_n$. Let $S_1,\dots,S_n$ be surjective information theoretic outlier scores. Then 
\begin{align*}
&(S_1*S_2 *\cdots * S_n) ((x_1,\dots,x_n)) \\
&= \sum_{j=1}^n S_j (x_j)  - \log \sum_{i=0}^{n-1}   \frac{(\sum_{j=1}^n S_j (x_j))^i}{i!}.
\end{align*}
\end{Theorem}
The proof is in the appendix.
Note that this result can be used to {\it define} an outlier score for product spaces (given that the components are independent). 

For $n=2$ we obtain the concise form
\begin{align*}
&(S_1* S_2) ((x_1,x_2)) =\\
& S_1(x_1) + S_2(x_2) - \log [1+ S_1(x_1) + S_2(x_2)]. 
\end{align*}

There is a good reason why the resulting score is not the sum of the outlier scores 
for the subsystems: otherwise the outlier score of a system that consists of multiple components would always be high.
The additional term is thus comparable to a correction term for multi-hypothesis testing. After all, each $S_j(x_j)$ can be seen as the log $p$-value of a statistical test. To understand the relevance of convolution, note that $S_1* \cdots * S_n$ may attain a high value either because at least one of the $S_j$ has been extremely high or because many $S_j$ are higher than expected. To be able to account for both types of exceptional events seems to be an advantage of 
convolution, compared to  more 'pragmatic' approaches where one just increases the threshold above which one considers an observation an outlier whenever a large set of metrics is taken into account.  

There are certainly a broad variety of possible definition of outlier scores. A very simple example would be, for instance, the 
normalized distance from the mean: if $X$ is real-valued, just define the z-score
\begin{equation}\label{eq:z}
\tilde{S}(x) :=\frac{|x-\Exp[X]|}{\sigma_X}
\end{equation}
 Then Chebyshev's  inequality also guarantees, to some extent,
that high outlier scores are rare due to $P_X\{ S(X) \geq c \} \leq 1/c^2$, although the probability doesn't decay  exponentially. However, for fixed unimodal distributions like Gaussians the z-score can be easily translated into 
the IT score Example~\ref{ex:quantile} (or equivalently Example~\ref{ex:rar})  via the non-linear monotonic transformation
\[
S(x) = -\log (1-{\rm erf}(\tilde{S}(x))).  
\] 
The following two sections suggest that information theoretic scores are particularly useful for discussing conceptual problems of root cause analysis -- even if z-score is often convenient because it avoids statistical problems of estimating cumulative distributions.

\section{Can a weak outlier {\it cause} a stronger one?
\label{sec:weak}}

In complex systems, outliers will certainly cause each other because, for instance, unexpectedly large values of one quantity may yield unexpectedly large 
values for quantities influenced by the former one. Within such a cascade of outliers one may want to identify  the `root cause' -- after first defining what this means.
It is tempting to consider the strongest outlier as the root cause. If the causal structure between the variables is known (in terms of a causal Bayesian network
\cite{Pearl:00}, for instance), one may alternatively search for the `most upstream' node(s) among those whose outlier scores are above a certain threshold, but 
since it is not clear where to set the threshold we will not discuss this any further.
To motivate a more principled approach to root cause analysis, let us first describe how choosing the largest outlier as root cause may fail.

\paragraph{Example with `bad' outlier score} Let us first consider problems with z-score \eqref{eq:z} applied to non-Gaussian variables. 
Assume two variables $X,Y$ are linked by the simple causal structure $X\rightarrow Y$ and the (deterministic) structural equation reads $Y=X^3$.
Assume symmetric distributions with $\Exp[X]=\Exp[Y]=0$. Then we have  $\tilde{S}(x)=x/\sigma_X$ and $\tilde{S}(y)= y/\sigma_Y=x^3/\sigma_Y$. Hence,
\[
\frac{\tilde{S}(y)}{\tilde{S}(x)} = x^2 \frac{\sigma_X}{\sigma_Y} \to \infty \quad  \hbox{  for } x\to \infty.
\]
This shows that large outliers cause a considerably even stronger outlier downstream, but here this is the result of that particular definition of the outlier scores.

\paragraph{Example with IT outlier score}
Let $X$ and $Y$ be Gaussian variables and $X$ influence $Y$ via the linear structural equation
\[
Y = X + N,
\]
where $X \sim \cN(0,1)$ and $N$ is an independent noise with $N\sim \cN(0,1)$. 
Recall that the z-score now is an IT score up to monotonic reparameteriation.
Let us set $X$ manually to the value $x=2$, which
is already a strong perturbation (values deviating from the mean by  $2$ standard deviations  or more occur with probability less than $0.05$). 
Assume further that $N$ attains the value $1$, which results in $Y$ attaining $y=2+1=3$, which is a quite strong outlier
because $y/\sigma_Y= 3/\sqrt{\sigma_X^2+\sigma_N^2}=3/\sqrt{2}\approx 2.2$. Hence,  $y$ is a stronger outlier than $x$ in terms of Rarity and Log Quantiles.  From an intuitive perspective, however, $x$ contributed more to $y$ being an outlier than $y$ itself. 
After all, $Y$ did not behave in an unexpected way, given the high value of its parent $X$; noise values of the size of the standard deviation are not particularly unlikely ($P\{N\geq 1\} \approx 0.16$). 

There is, however, a sense in which 
an outlier is unlikely to cause a significantly stronger one with respect to an IT score:
\begin{Lemma}[relations between outlier scores]\label{lem:Scond}
For any $\Delta \in \R_0^*$ we have
\[
P \{ S(Y)\geq c+\Delta | S(X)\geq c \} \leq e^{-\Delta}.
\]
for almost all $c\in S(\cX)$.
\end{Lemma}
\begin{proof} We have:
\begin{align*}
& P  \{ S(Y)\geq c+\Delta | S(X)\geq c \}  \\
= & P  \{ S(Y)\geq c+\Delta, \,S(X)\geq c \} / P\{S(X) \geq c \}  \\
 \leq & P\{ S(Y) \geq c +\Delta \} / P\{S(X) \geq c \}  \leq \frac{e^{-c-\Delta}}{e^{-c}},
\end{align*}
where we used that $P\{S(X) \geq c\}=e^{-c}$ holds for all $c\in S(\cX)$. 
\end{proof}
Note that Lemma~\ref{lem:Scond} holds independently of whether $X$ is the cause and $Y$ the effect or vice versa: Plugging an outlier $x$ with $S(x)\geq c$ into any causal mechanism
$P_{Y|X}$, is unlikely to cause an effect $y$ whose outlier score is significantly larger than $c$. On the other hand, given that the effect $x$ has outlier score
$c$ at least, it is unlikely that the cause $y$ plugged into the mechanism $P_{X|Y}$ had an outlier score significantly larger than $c$.

\section{Conditional outlier scores \label{sec:co}}

The previous section discussed to what extent the value attained by a node  
is unexpected {\it given the value of its parents}. The intuition behind this question implicitly referred to
 a notion of {\it conditional} outlier scoring, which we now introduce formally by using the the conditional distribution of $Y$, given its parents. Defining outlier scores by conditional distributions given some background information is certainly not novel \cite{Song2017}. The present paper, however,  emphasizes that
analyzing the {\it causal contribution} of each node to an outlier requires `causal conditionals', that is, conditionals of a node given its parents. If an anomaly is likely, given its parents, we would consider the anomaly as `caused by' the latter and not caused `by the node itself'. In contrast, whether the anomaly is likely {\it given its children}, is irrelevant for root cause analysis.

\subsection{Two variable case}

We first restrict the attention to a causal structure with just two variables $X,Y$ where $X$ is the cause of $Y$.
Generalizing this to the case where the target variable
is influenced by multiple causes will be straightforward since none of the results requires $X$ to be one-dimensional.  

\begin{Definition}[Conditional outlier score]\label{def:condscore}
Let $X,Y$ be random variables with range $\cX,\cY$, respectively, and causally linked by $X\rightarrow Y$. 
A measurable function $S_{Y|X} :\cY \times \cX\to \R^+_0$ is called conditional IT outlier score if for all $x\in \cX$,
$S_{Y|X} (\cdot,x)$ is an IT score with respect to the conditional distribution $P_{Y|X=x}$.
\end{Definition}
It will be convenient to write $S(y|x)$ instead of $S_{Y|X}(y,x)$.  The following property supports the 
view that conditional outliers quantify to what extent an unlikely event happened at the particular node under consideration and not at its parents:
\begin{Lemma}[independence of surjective scores] If $S_{Y|X}(\cdot,x)$ is a surjective IT score for every $x\in \cX$, then
$S_{Y|X}(Y,X) \independent X.$
\end{Lemma}
The lemma is an immediate consequence of Lemma~\ref{lem:sur}  since $S_{Y|X}(Y|_{X=x},x)$ has the same density $s \mapsto e^{-s}$ for all $x\in \cX$. 

We have emphasized that conditional outlier score is supposed to measure to what extent something unexpected happened at that particular node. This perspective is facilitated when a conditional outlier score is -- a priori -- defined by an outlier of the unobserved noise that corresponds to that node in an FCM: 
\begin{Definition}[FCM for conditional score] \label{def:fcm} A conditional IT outlier score $S_{Y|X}$ is said to have an FCM
if there is an FCM $Y=f(X,N)$ with range $\cN$ for $P_{Y|X}$ and an IT score $S_N:\cN\to \R_0^+$ such that
$N$ is a function of $(X,Y)$ and $S_{Y|X}(Y,X)=S_N(N)$.
\end{Definition}
\begin{Lemma}[existence of FCM for a conditional score]\label{lem:fcmscore}
A conditional IT outlier score $S_{Y|X}$ has an FCM if and only if 
\[
X \independent S_{Y|X}(Y|X). 
\]
\end{Lemma}
The proof is in the appendix. 

\subsection{Causal Bayesian Network}

Since Definition~\ref{def:condscore} did not put any restriction on the range $\cX$,  we can thus assume $X$ to be the multivariate variable that consists of all parents $PA_j$ of a variable $X_j$ in a Bayesian network
and obtain, in a canonical way, the conditional score
\[
S_{X_j|PA_j}: \cX_j \times \cPA_j  \rightarrow \R_0^+,
\]
where $\cPA_j$ denotes the range of $PA_j$.

Whenever the outlier scores have an FCM in the sense of Definition~\ref{def:fcm}, they are independent
random variables. To decide whether an observed $n$-tuple $\bx=(x_1,\dots,x_n)$ defines an outlier event, we
can then again use the notion of convolution:

\begin{Definition}[convolution of cond.~scores] 
Given conditional IT outlier scores $S_{X_1|PA_1},\dots, S_{X_n|PA_n}$ for some Bayesian network that are independent as random variables, let
$f$ be given by
\[
f(x_1,\dots,x_n):= \sum_{j=1}^n S(x_j|pa_j).
\]
Then the convolution score is defined via
\[
(S_{X_1|PA_1} * \cdots * S_{X_n|PA_n})(\bx) := S^f(\bx).
\] 
\end{Definition}
By straightforward adaption of the proof of Lemma~\ref{thm:conv} we obtain for this case:
\begin{Theorem}[convolution is subadditive]\label{thm:condconv}
Whenever  $S_{X_1|PA_1},\dots,S_{X_n|PA_n}$ are  independent random variables having the density $p(n)=e^{-n}$, we have
\begin{align*}
& (S_{X_1|PA_1} * \cdots * S_{X_n|PA_n})(\bx)  \\
= & \sum_{j=1}^n S(x_j|pa_j) - \log \sum_{i=0}^{n-1}   \frac{(\sum_{j=1}^n S(x_j|pa_j))^i}{i!}.
\end{align*}
\end{Theorem}

\section{Attributing outliers of a target variable to its  ancestors \label{sec:attr}}

In applications one may often be interested in outliers
of one target variable, say $X_n$, and wants to attribute this outlier to its ancestor nodes, or more precisely,
to quantify to what extent each node {\it contributed} to the observed outlier  $x_n$.
Assume, without loss of generality, that $X_n$ is a sink node, that is, that it has no descendants and that we are given an FCM of all the variables
$X_1,\dots,X_n$ with noise variables $\bN =(N_1,\dots,N_n)$ with outlier scores $S_{X_j|PA_j}$ given by $S_{N_j}$ for $N_j$.
By recursively applying \eqref{eq:fcm} we can write $X_n$ as deterministic function
of all noise variables.
Therefore,
\begin{eqnarray}
- \log P\{ S(X_n) \geq S(x_n) \}  &= & S(x_n) \nonumber\\
\log P\{  S(X_n) \geq S(x_n) | \bN=\bn \} &=& 0.  \label{eq:logscore}
\end{eqnarray}
The first equality is due to Definition~\ref{def:score} and the second because $X_n$ is a deterministic  function of $\bN$.
We can thus change  \eqref{eq:logscore} step by step from $0$ to $S(x_n)$ by dropping more and more
of the components $N_j$ from the vector $\bN$. It is tempting to quantify the contribution of each node $j$ by the change caused by dropping $N_j$, but this value depends on the ordering of nodes. Fortunately, there is the concept of
Shapley values from cooperative game theory that solves the problem of order-dependence by averaging over all possible orderings in which elements are excluded \cite{Shapley1953}.\footnote{More recently, Shapley values have also been used in the context of quantifying feature relevance \cite{Datta2016}.} We rephrase this concept for our context. First, we define the contribution of $j$, given some
set $T\subset \{1,\dots,n\}=:U$ not containing $j$, by
\begin{eqnarray*}
c_\bx(j|T) &:=& \log P\{  S(X_n) \geq S(x_n) | \bn_{T\cup \{j\}} \}  \\
   &-& \log P\{  S(X_n) \geq S(x_n) | \bn_T \}.  
\end{eqnarray*}
We then define the Shapley contribution of $j$ by
\begin{equation}\label{eq:shapley}
c^{Sh}_\bx(j):=  \sum_{T\subset U\setminus \{j\}}  \frac{1}{n {n-1 \choose |T|}}   c_\bx (j | T).
\end{equation}
The following result is a straightforward application of Shapley's idea and follows because
$c_\bx^{Sh}$ is just defined by symmetrization over all possible orderings.
\begin{Theorem}[decomposition of target outlier score]\label{thm:attr}
The outlier score of any variable decomposes into the contribution of each of its ancestors:
\[
S(x_n) = \sum_{j=1}^n c^{Sh}_\bx(j). 
\]
 \end{Theorem}
 Note that this contribution can be negative which makes perfectly sense:
one value being extreme can certainly weaken the outlier of the target, that is, a more common value at that node would have made the outlier even stronger.

\section{Experiments\label{sec:exp}}
The goal of the following experiments were to explore whether the concept of {\it conditional} outliers is better at finding the root cause
of outliers than the naive approach of comparing unconditional outlier scores.  The difference between IT scores and other scores
did not play a role for these experiments because we work with simple Gaussian probabilistic models to keep the focus on the difference between
conditional versus unconditional score. The main purpose of Sections~\ref{sec:it} and \ref{sec:weak} was to provide a solid basis for
outlier scoring in the first place before discussing the difference between conditional versus unconditional scores.

\subsection*{Synthetic Datasets}

Here we injected perturbations to randomly chosen nodes in random DAGs. Each node had a $15\%$ chance to be perturbed  by a change of the structural equation from
\begin{align}
	X & = f(PA_X) + N_X \label{eq:anm}\\
	\hbox{ to } \quad \quad X &= f(PA_X) + N_X + \lambda \sigma_{N_X} \label{eq:anmchanged},
	\end{align}
where $\lambda$ is a parameter which we refer to as the perturbation strength from here onwards, and $\sigma_{N_X}$ is the standard deviation of the noise $N_X$ (by restricting the attention to additive noise models \cite{Hoyer} in
\eqref{eq:anm} $\sigma_X$ defines a natural scale to define perturbation strength). 
We have considered the node(s) at which the perturbation was injected as the `root cause(s)' of the outlier. 

We have generated DAGs with $18$ nodes (to guarantee sufficient sparsity we have decided to chose $3$ parentless nodes and $15$ nodes for which the number $n$ of parents is randomly chosen with probabilities decaying inverse proportional to $n$).  The function $f$ is chosen as follows. 
		(1) With 20\% chance, linear with random coefficients in $\mathcal{U}(-1,1)$ and no intercept.
(2) With 80\% chance, a random non-linear function generated by a neural network with one hidden layer, random weights in $\mathcal{U}(-5, 5)$, a random number of neurons in $[2, 3, \dotsc, 100]$, and the sigmoid activation function.
This way, we guarantee to have some linear and, in practice more common, complex non-linear relationships.
The noise $N_X$ is chosen Gaussian or uniformly distributed with randomly chosen width.

\begin{figure}[!htbp]
	\centering
	\includegraphics[width=\columnwidth]{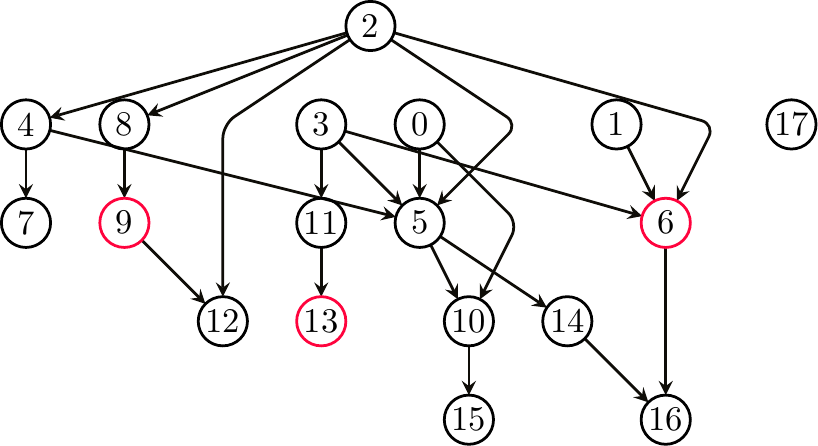}
		\caption{A randomly generated synthetic causal graph with root causes marked in red.}
	\label{fig:causal-graph}
\end{figure}

For each node in the causal graph, we sampled 2000 observations. 
We inferred a node to be perturbed whenever its unconditional score (this defined the baseline)   or its conditional score (our proposal) exceeded some threshold. 
We used z-score \eqref{eq:z} as unconditional score and
\begin{equation}\label{eq:condz}
\tilde{S}(x_j|pa_j):=  \frac{|x_j - \Exp[X_j|pa_j]|}{\sigma_{X_j|pa_j}},
\end{equation}
as conditional score,
where the conditional expectation $\Exp[X_j|pa_j]$ was estimated via a simple linear regression model. 
Again, this is a valid conditional IT outlier score with respect to a linear Gaussian probabilistic model. The value
$x_j - \Exp[X_j|pa_j]$ is then just the value of the Gaussian noise and the conditional standard deviation
$\sigma_{X_j|pa_j}$ is the standard deviation of the noise.
To avoid estimating cumulative tail distributions (which requires large sampling) we applied the simple z-score
also to our non-Gaussian variables, also to
 test the robustness of our method with respect to violations of distributional assumptions. 

First we examined how perturbation strength $\lambda$ affects the performance of outlier scores for a fixed random DAG shown in Figure~\ref{fig:causal-graph}. The nodes marked in red are the root causes of outliers, that is, the perturbed nodes. We estimate the functional relationship between a node and its parents by a random forest with 100 trees and a maximum depth of 7 on the training set. In the test set, we compare ROC curves of both scores at various values of perturbation strength. The results are shown in  Figure~\ref{fig:roc-curves}. 

\begin{figure*}[!htbp]
	\centering
	\begin{minipage}[t]{0.23\textwidth}
		\centering
		\includegraphics{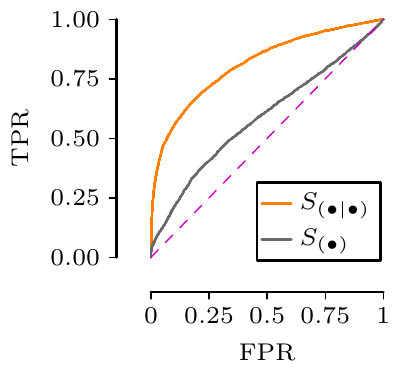}
	\caption*{$\lambda=2$}
	\end{minipage}
	\hspace{7mm}
	\begin{minipage}[t]{0.23\textwidth}
		\centering
		\includegraphics{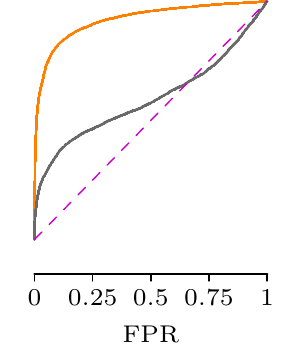}
		\caption*{$\lambda=3$}
	\end{minipage}
	\hfill
	\begin{minipage}[t]{0.23\textwidth}
		\centering
		\includegraphics{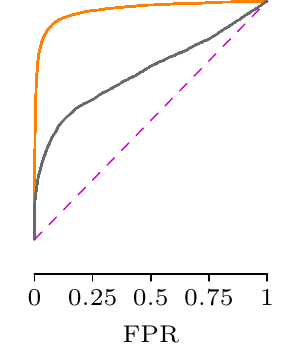}
		\caption*{$\lambda=4$}
	\end{minipage}
	\hfill
	\begin{minipage}[t]{0.23\textwidth}
		\centering
		\includegraphics{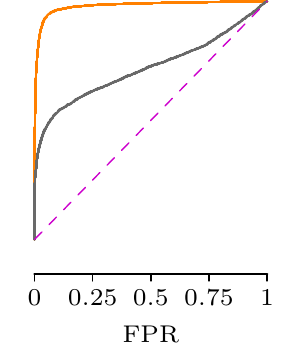}
		\caption*{$\lambda=5$}
	\end{minipage}
	\caption{ROC curves for various perturbation strengths ($\lambda$) on the root causes of the graph in Figure.~\ref{fig:causal-graph} for conditional (orange) and unconditional (black) outlier score.}
	\label{fig:roc-curves}
\end{figure*}

Already at $\lambda=2$, we observe that conditional outlier score is much better at identifying root causes than its unconditional counterpart. Unlike unconditional outlier score, the true positive rate (TPR) of conditional outlier score shoots up quickly with a slight increase in false positive rate (FPR). As the perturbation strength increases, the conditional outlier score identifies most of the root causes (true positives) only at a rather small cost of allowing very few false positives, in contrast to unconditional outlier score.

	\begin{figure*}[h]
		\centering
		\begin{minipage}[c]{0.2\textwidth}
			\includegraphics[width=\columnwidth]{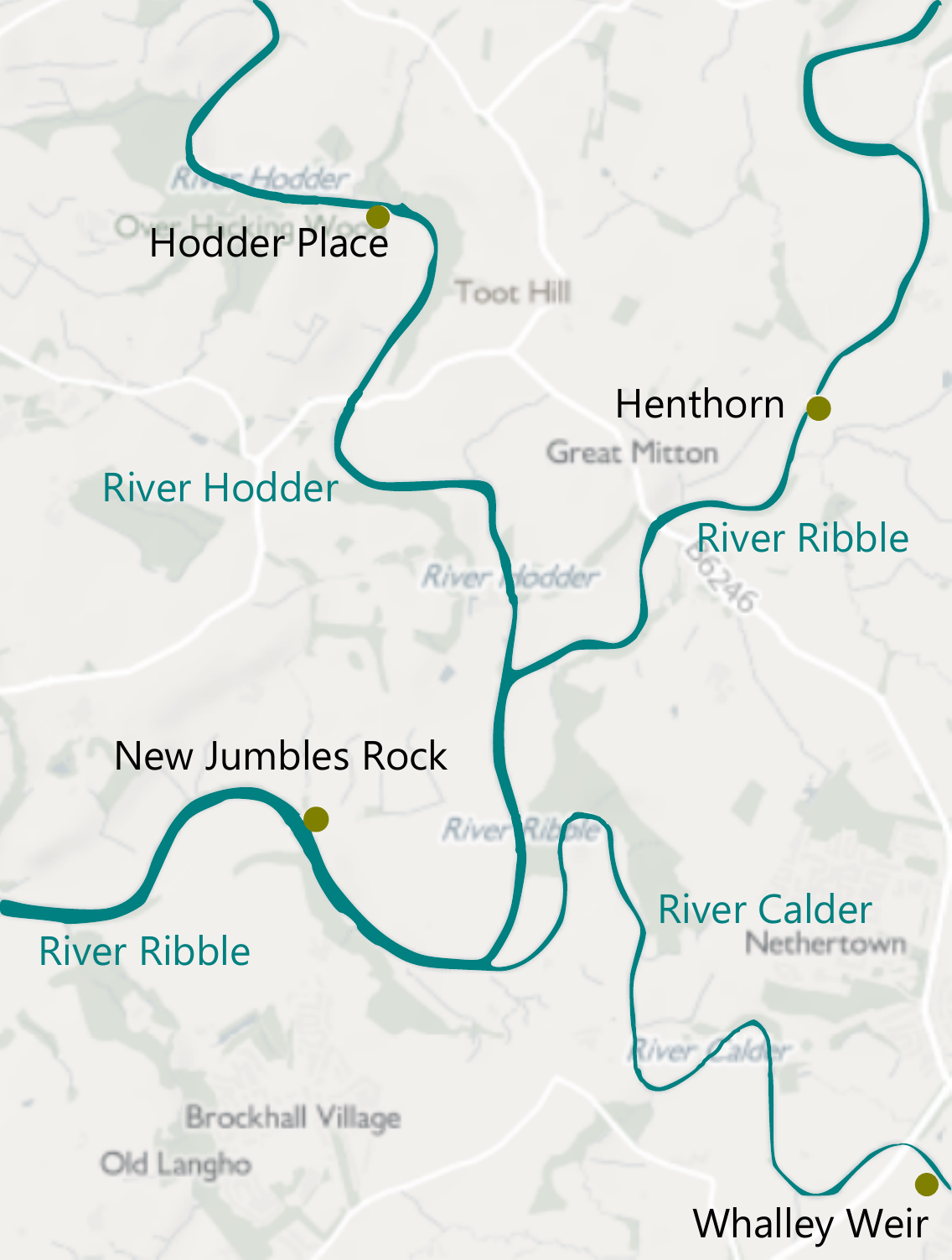}\\
		\end{minipage}
		\hfill
		\begin{minipage}[c]{0.78\textwidth}
			\includegraphics[width=\columnwidth]{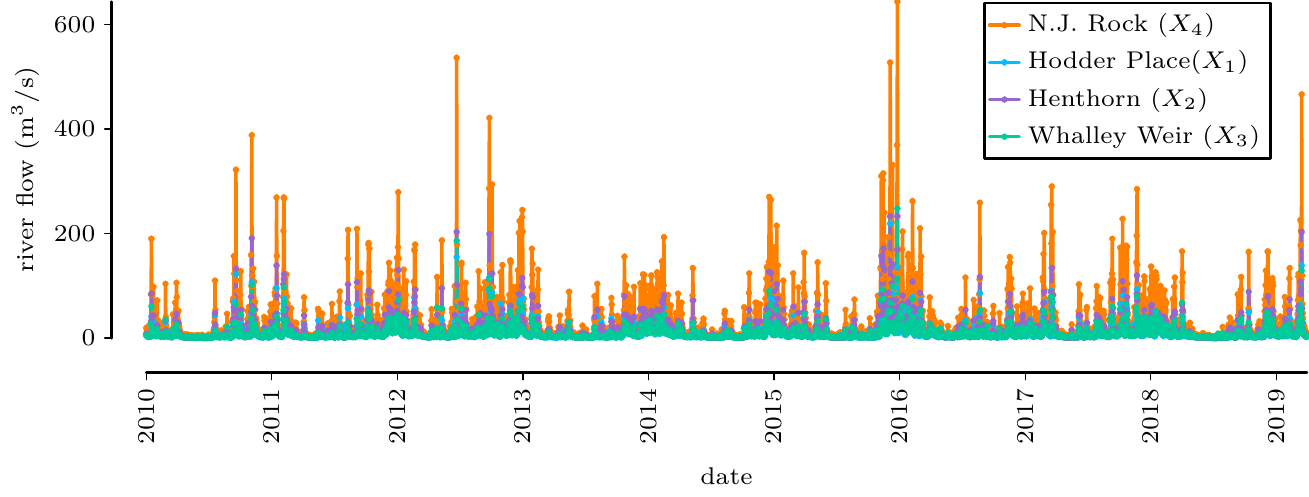}
				\end{minipage}
		\caption{(left) Map of measurement stations located before and after the confluence of rivers Hodder, Calder, and Ribble. (right) River flows from 1 Jan. 2010 till 31 Mar. 2019 at those measurement stations.}
		\label{fig:river-map-and-flow}
	\end{figure*}

\begin{figure}[!htbp]
	\centering
	\includegraphics[width=\columnwidth]{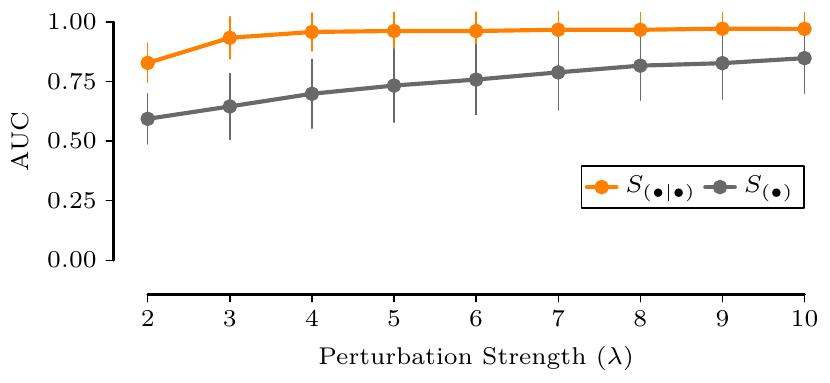}
	\caption{Area under ROC curve at various $\lambda$.}
	\label{fig:auc}
\end{figure}

We sampled 190 causal graphs, and report average area under ROC curve (AUC) together with the sample standard deviation at various values of perturbation strength in Figure~\ref{fig:auc}. We see that the results corroborate the findings from our previous experiment. The performance of both scores improve with the increase in the perturbation strength. The conditional outlier score, however, consistently performs better than unconditional outlier score.

\subsection*{Real Dataset} 
Here we considered daily river flow, measured in cubic metres per second, at various locations along three rivers in England.\!\footnote{\url{https://environment.data.gov.uk/hydrology/explore}} In particular, we consider the daily river flow at the New Jumbles Rock ($X_4$) measurement station, which is located right after the confluence of rivers Hodder, Calder, and Ribble. Furthermore, we consider three other measurement stations upstream of the confluence: Hodder Place ($X_1$) along River Hodder, and Henthorn ($X_2$) along River Ribble, and Whalley Weir ($X_3$) along River Calder (see Figure~\ref{fig:river-map-and-flow}). As river flow downstream of the confluence is the result of river flows upstream, we can reasonably assume the following causal graph\!\footnote{Note, however, that downstream water levels can in principle also influence upstream ones by slowing down the flow.}, up to unobserved common causes like weather conditions (e.g. precipitation and temperature), which we haven't accounted for since this required deeper domain knowledge like time delay of these influences.

	\begin{figure}[H]
		\centering
		\includegraphics[width=0.35\columnwidth]{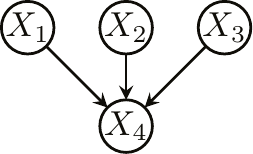}
			\end{figure}

	\begin{figure}[h]
		\centering
		\includegraphics[width=0.9\columnwidth]{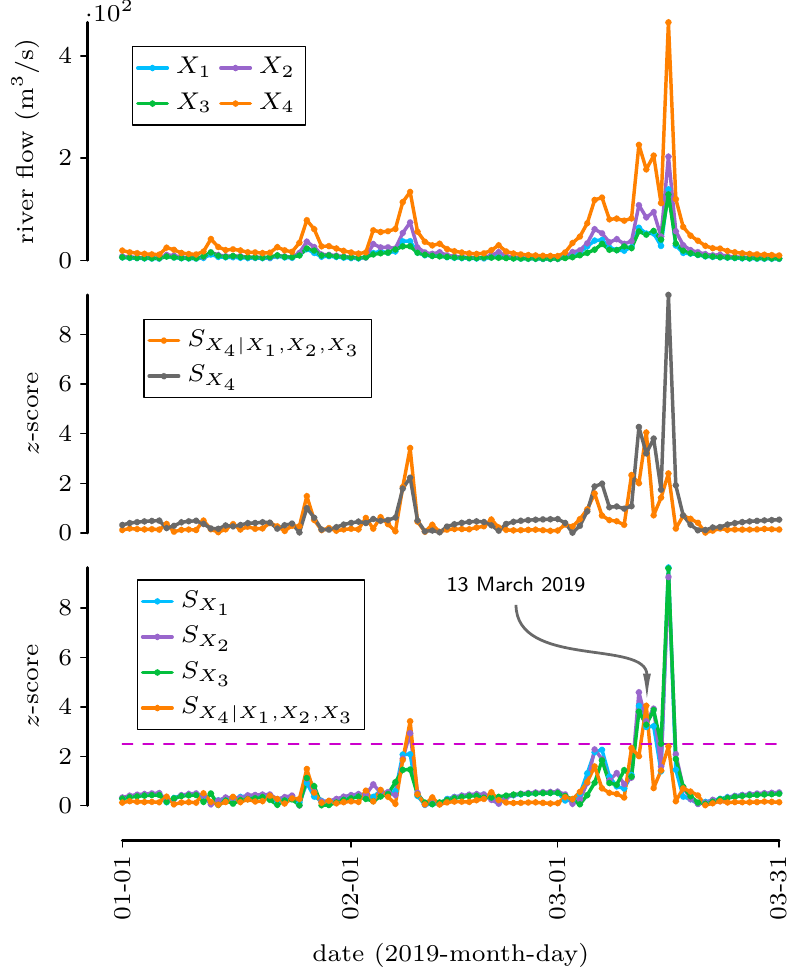}
		\caption{In the test set, we show (top) river flows at measurement stations, (middle)  conditional vs. unconditional outlier scores at the New Jumbles Rock station ($X_4$), and (bottom) conditional outlier scores at all measurement stations. The dashed magenta line represents the $z\!=\!2.5$ threshold.}
		\label{fig:test-outlier-scores}
	\end{figure}

	For each measurement station, we take daily river flows, measured at 9:00 a.m.~every day, from 1 January 2010 till 31 March 2019 (shown in Figure~\ref{fig:river-map-and-flow}). If a daily measurement is missing in one of the stations, we remove the corresponding daily measurements from all other stations. Consequently we end up with 3357 measurements. We use 3267 measurements before 1 January 2019 for fitting the structural equations, and the remaining 90 measurements, from the year 2019, for testing.

	Assuming that the underlying data generation process follows a linear additive noise model, we can represent the causal graph above with the structural equations $X_j=N_j$ for $j=1,\dots,3$ and
	\begin{align*}
		X_4 & = \alpha X_1 + \beta X_2 + \gamma X_3 + N_4 \; ,
	\end{align*}
	where $N_i$ is the noise variable corresponding to the variable $X_i$. 
	For the z-scores, we again estimated means and variances of variables $X_1$, $X_2$, and $X_3$. For $X_4$, we apply ordinary least squares linear regression, and again estimated the mean and the variance of the resulting noise
	and estimated condional z-score \eqref{eq:condz}.
		As histograms of noise variables suggest a roughly Gaussian shape, it again suffices to use a simplified outlier score such as the normalized distance from the mean. 	
	First we compare conditional against unconditional outlier scores on the test set at the New Jumbles Rock measurement station ($X_4$) in Figure~\ref{fig:test-outlier-scores} (middle).\!\footnote{Note that conditional and unconditional outlier scores are the same for $X_1$, $X_2$, and $X_3$ because they do not have any parents in the causal graph.} We observe that unconditional outlier score is not only often larger but also more sensitive to changes than its conditional counterpart. Consider the two scores on 16 March 2019, for instance. On this day, river flow is the largest at the New Jumbles Rock station. This is not surprising, however, given the high volume of water upstream. The conditional outlier score takes that into account, and is hence a mere $2.39$. The unconditional outlier score, on the other hand, is a massive $9.58$.

	Next we investigate the root cause of outlying measurements. To this end, in Figure~\ref{fig:test-outlier-scores}, we plot conditional outlier scores of measurements on the test set for all nodes in the causal graph. We examine one day in particular here. On 13 March 2019, for instance, we observe that conditional outlier scores of measurements at all the nodes are very high. This suggests us that something unlikely happened in all those nodes on that day before 9:00. One possible explanation is the heavy rainfall early morning in Lancashire county where the rivers are located in.
	
	Overall the experiments show that conditional outlier scores may differ significantly from unconditional ones. Consequently, even simple linear regression models sometimes suffice to {\em explain} outliers as resulting from upstream outliers. 
After all, every event for which large conditional outlier score come with small conditional scores show that the FCM fits well for that particular event. Also our simulation experiments suggest that values that have low conditional outlier score but high unconditional score are just downstream effects.	
Of course we have no ground truth regarding the `root causes' -- it is not even clear how to define it without the concept of conditional outliers
introduced here.

\section{Discussion}
We have proposed a way to quantify to what extent an observation is an unlikely event {\it given the value of its causes}
in a scenario where the causal structure is known. 
This way, we can attribute rare events to mechanisms associated with specific nodes 
in a causal network.
 This attribution is formally provided by two different main results. First, Theorem~\ref{thm:condconv} relates the outlier score
of a collective event $\bx$ to conditional outlier scores of each nodes. Second, Theorem~\ref{thm:attr} quantitatively attributes the outlier score of one target variable to unexpected behavior of its ancestors. Our results do not solve the hard problem of causal discovery \cite{Spirtes1993}, that is, the question of how to obtain the causal DAG in the first place. Instead, it aims at providing a framework to formally talk about attributing
rare events to root causes when the causal DAG is given and the structural equations are either given or inferred from data .\footnote{Structural equations do not uniquely follow from observed joint distributions even when the DAG is given \cite{Pearl:00}, but they can be inferred subject to appropriate assumptions, e.g., additive noise \cite{Hoyer}.}  

To avoid possible misunderstandings, we emphasize that anomalies may certainly occur because 
the causal DAG, the structural equations, or the corresponding noise distributions do not hold for that  
particular statistical realization. The fact that our computation of conditional outlier score is based on
assuming a fixed DAG, structural equation, and noise distribution does not mean that we exclude these changes.
The conditional outlier score should rather be considered as the negative logarithm of a p-value that can be used for {\it rejecting} the hypothesis that the given structural assumptions still hold for the given anomalous event.

{\bf Acknowledgements:} Thanks to Atalanti Mastakouri for comments on an earlier version of the manuscript.

\section{Appendix}

\subsection{Proof of Theorem~\ref{thm:conv}}

For surjective outlier scores we can without loss of generality assume $\cX_j=[0,\infty)$ with the density 
$p_j(x_j)= e^{-x_j}$, see Lemma~\ref{lem:sur}. Set $s:=  \sum_{j=1}^n S_j(x_j)$. 
The density on $\cX$ as the positive cone of $\R^n$ is then given by
\[
p(s_1,\dots,s_n) =  e^{-\sum_{j=1}^n s_j}.
\]
We have to integrate this density over the simplex $s_j\geq 0$ with $\sum_j s_j \leq s$ and thus obtain
\begin{align}
& P \left\{\sum_{j=1}^n S_j(X_j) \geq s \right\} \nonumber \\
&= \int_{\sum_{j=1}^n s_j \geq s}  e^{-\sum_{j=1}^n s_j } ds_1 \cdots ds_n.\label{eq:integral}
\end{align}
We now perform the substitution $(s_1,\dots, s_n) \mapsto (s_1,\dots,s_{n-1},s')$ with 
$s':=\sum_{j=1}^n s_j$ with Jacobian determinant $1$ and obtain for \eqref{eq:integral}:
\begin{align*}
\int_s^\infty e^{-s'}  \int_{\sum_{j=1}^{n-1} s_j \leq s'} ds_1\cdots ds_{n-1}    ds'
\end{align*}
The inner integral is just the volume of the simplex in $\R^{n-1}$ given by $s_j \geq 0$ and $\sum_{j=1}^{n-1} s_j \leq s'$, which reads $(s')^{n-1}/(n-1)!$.
We thus obtain
\begin{align*}
& P \left\{\sum_{j=1}^n S_j(X_j) \geq s\right\} 
= \int_s^\infty  e^{-s'} \frac{(s')^{n-1}}{(n-1)!} ds'.
\end{align*}
 Using  \cite{Bronshtein1973}
\begin{eqnarray*}
\int y^n e^{cy} dy &=& e^{cy} \sum_{j=0}^k (-1)^{k-j} \frac{k!}{j! c^{k-j+1} } y^j, 
\end{eqnarray*}
with $y=s'$, $c=-1$, and $k=n-1$ we obtain
\begin{align*}
& P \left\{\sum_{j=1}^n S_j(X_j) \geq s \right\}\\
&=  \left.  e^{-s'}  \sum_{j=0}^{n-1} (-1)^{n-1-j} \frac{1}{j! (-1)^{n-j} } (s')^j \right|_s^\infty\\
&= e^{-s} \sum_{j=0}^{n-1}   \frac{s^j}{j!}. \\   
\end{align*}
Inserting $\sum_{j=1}^n S_j(x_j)$ and taking the logarithm of the probability completes the proof.

\subsection{Proof of Lemma~\ref{lem:fcmscore}}

The `only if' part is obvious since $N$ is by definition of FCMs independent of $X$. To show the converse direction,
let $Y=f(X,N)$ be an FCM for $P_{Y|X}$ for some noise variable $N$. 
 Define a modified FCM 
 \begin{equation}\label{eq:mfcm}
  Y=\tilde{f}(X,\tilde{N}),
 \end{equation}
 with noise variable $\tilde{N}:= (N_1,N_2)$, where $N_1$ has the same range as $N$ and $N_2:=S(Y|X)$.
 Define an outlier score for $\tilde{N}$ by
 \[
 S_{\tilde{N}}(\tilde{N}):= N_2.
 \]
 To show that it satisfies the distributional requirements of an IT score we  need to show $P_{X,Y}\{ S(Y|X) \geq S(y|x) \} = e^{-S(y|x)} $ for $P_{X,Y}$-almost all
 $(x,y)$. By assumption we have $P_{Y}\{ S(Y|x) \geq S(y|x) \} =  e^{-S(y|x)} $ for $P_{X,Y}$-almost all $(x,y)$. Due to the independence of $X$ and $S(Y|X)$ 
 the distribution of $S(Y|X)$ coincides with the distribution of $S(Y|x)$ for $P_X$-almost all $x$. Thus 
 we have $P_{Y} \{  S(Y|x) \geq S(y|x) \} = P_{X,Y} \{  S(Y|X) \geq S(y|x) \}$ for  $P_{X,Y}$-almost all $(x,y)$, which shows that $S_{\tilde{N}}$ is an IT score.
 
 Then, define a joint distribution of $X,N_2,N_1$ by
 \[
 P_{N_1,N_2,X} :=  P_{N_1|S(Y|X)} P_{S(Y|X)} P_X,
 \]
with
\[
P_{N_1|S(Y|X)} := P_{N|S(Y|X)}.
\]
We observe that $\tilde{N} \independent X$ due to weak contraction because
$N_1 \independent X\,| N_2$ by construction and  $N_2 \independent X$ by assumption.  
 Further, $P_{N_1}$ has the same distribution as $P_N$ since $P_{N_1|S(Y|X)}= P_{N|S(Y|X)}$. Hence,
 \eqref{eq:mfcm} is also an FCM for $P_{Y|X}$.  Moreover, we have 
  \[
 S(\tilde{N}) = N_2 = S(Y|X)
 \]
 by construction.  Hence we have constructed an FCM for $S(Y|X)$.


\end{document}